\documentclass[11pt,a4paper]{article}

\usepackage[margin=1in]{geometry}
\usepackage{amsmath,amssymb,amsthm}
\usepackage{booktabs}
\usepackage{array}
\usepackage{multirow}
\usepackage{hyperref}
\usepackage{xcolor}
\usepackage{graphicx}
\usepackage{enumitem}
\usepackage{natbib}
\usepackage{microtype}
\usepackage{mathtools}

\hypersetup{
  colorlinks=true,
  linkcolor=blue!60!black,
  citecolor=blue!60!black,
  urlcolor=blue!60!black
}

\theoremstyle{plain}

\newtheorem{proposition}{Proposition}

\theoremstyle{definition}
\newtheorem{definition}{Definition}
\theoremstyle{remark}

\newcommand{\ERS}{\textsc{ers}}
\newcommand{\MOM}{\textsc{mom}}
\newcommand{\HMM}{\textsc{hmm}}
\newcommand{\DQN}{\textsc{dqn}}
\newcommand{\POMDP}{\textsc{pomdp}}

\newcommand{\vtrap}{\Delta v_{\mathrm{trap}}}
\newcommand{\tsec}{\Delta t_{\mathrm{sector}}}
\newcommand{\bbrake}{\Delta b_{\mathrm{brake}}}
\newcommand{\sspeed}{\sigma^2_{\mathrm{speed}}}
\newcommand{\dthrottle}{\delta_{\mathrm{throttle}}}
\newcommand{\zaero}{z_{\mathrm{aero}}}
\newcommand{\bt}{b_t}
\newcommand{\xt}{x_t}
\newcommand{\R}{\mathbb{R}}

\title{\textbf{Opponent State Inference Under Partial Observability:\\
An HMM--POMDP Framework for 2026 Formula~1 Energy Strategy}\\[0.5em]
\large Kalliopi Kleisarchaki\\
\normalsize Independent Researcher}

\date{%
\small
v1 (1~March 2026): Submitted to arXiv. Five observables, 30 hidden states.\\
v1.5 (5~March 2026): Sixth observable $\dthrottle$ added. ResearchGate pre-registration before Melbourne
(\url{https://www.researchgate.net/publication/401621952}).\\
v2 (post-Melbourne 8~March 2026; arXiv 9~March 2026): 40-state architecture.
$L_{\mathrm{harvest}}$/$L_{\mathrm{derate}}$ split formalised.
Pre-registered on ResearchGate before China (Race~2).\\
v3 (15~May 2026): Editorial corrections.\\[0.5em]
\textit{Pre-Registration active. Empirical calibration: Australian GP (Race~1, 8~March 2026) onwards.}
}

\begin{document}

\maketitle
\thispagestyle{empty}

\begin{abstract}
The 2026 Formula~1 technical regulations introduce a fundamental change to energy strategy:
under a 50/50 internal combustion engine / battery power split with unlimited regeneration
and a driver-controlled Override Mode (hereafter \MOM{}), the optimal energy deployment
policy depends not only on a driver's own state but on the \emph{hidden state} of rival cars.
This creates a Partially Observable Stochastic Game that cannot be solved by single-agent
optimisation methods.

We present a tractable two-layer inference and decision framework.
The first layer is a \textbf{40-state} Hidden Markov Model (\HMM{}) that infers a probability
distribution over each rival's \ERS{} charge level---now decomposed into four modes
($H$, $M$, $L_{\mathrm{harvest}}$, $L_{\mathrm{derate}}$)---Override Mode status, and tyre
degradation state from six publicly observable telemetry signals.
The second layer is a Deep Q-Network (\DQN{}) policy that takes the \HMM{} belief state as
input and selects between energy deployment strategies.
We formally characterise the \emph{counter-harvest trap}---a deceptive strategy in which a
car deliberately suppresses observable deployment signals to induce a rival into a failed
attack---and show that detecting it requires belief-state inference over \emph{both} ERS level
and the harvest/derate sub-mode.

The $L_{\mathrm{harvest}}$/$L_{\mathrm{derate}}$ split is the central architectural advance
of v2: both sub-modes produce depressed $\vtrap$ and can co-occur with $\zaero{} = 1$ in an
activation zone, but they carry \emph{opposite} strategic implications.
$L_{\mathrm{harvest}}$ means the rival is building a hidden reserve---the trap condition.
$L_{\mathrm{derate}}$ means the rival is at their physical limit---a genuine attack
opportunity.
The sixth observable $\dthrottle$ (super-clipping duration fraction, added in v1.5) provides
clean separation between these sub-modes; v2 elevates this separation to the \emph{state level}
rather than handling it as a mixed emission.

On synthetic races generated under the 40-state parametric model, the \HMM{} achieves
\textbf{96.8\%} ERS-level accuracy (random baseline: 25\%), correctly classifies
$L_{\mathrm{harvest}}$ vs.\ $L_{\mathrm{derate}}$ with \textbf{89.4\%} accuracy (random
baseline: 50\%), and detects counter-harvest trap conditions with \textbf{96.3\%} recall.
Pre-season analysis indicates circuit-dependent recharge availability (1.0$\times$ to
2.2$\times$ per lap) as the primary confound for trap detection; Melbourne represents the
hardest-case validation environment.
Baum-Welch calibration on 2026 race telemetry begins with the Australian Grand Prix
(8~March~2026).
Code will be made publicly available upon acceptance.
\end{abstract}

\newpage

\newpage

\section{Introduction}
\label{sec:intro}

Formula~1 has always been a game of imperfect information.
A driver knows their own tyre temperature, fuel load, and energy reserves exactly;
they observe rival positions through timing screens and team radio;
but the \emph{internal state} of a rival's car---how much battery charge they carry,
whether their Override Mode is available or spent, how degraded their tyres are---remains
hidden.
For most of F1's history this did not matter strategically: energy deployment was constrained
enough that the optimal decision could be computed from own-car state alone.

The 2026 regulations change this \citep{FIA2025}.
The 50/50 power split between internal combustion engine and MGU-K battery places the
battery under permanent demand, making \ERS{} charge level a first-order strategic variable
every sector.
Override Mode is a proximity-gated energy award: when a driver is within one second of the
car ahead at the designated detection point, they earn an extra 0.5\,MJ of deployable energy
on the following lap.
Active Aero replaces DRS entirely: the adjustable rear wing is removed and replaced with a
driver-controlled binary aero system available only in designated activation zones.
Crucially, a car can run Active Aero in straight-line mode \emph{while harvesting}---the
mechanical downforce reduction provides speed independent of battery deployment.
This is the physical basis for the counter-harvest trap (Section~\ref{sec:trap}).

The optimal burn-or-harvest decision in any sector is now a function of the \emph{joint
state} of both cars, not just one's own.

\subsection*{Contributions}

This paper makes three contributions:

\begin{enumerate}[label=\arabic*.]
\item \textbf{Problem formalisation.}
We formally model 2026 F1 energy management as a Partially Observable Stochastic Game
(\textsc{posg}) and define a tractable \POMDP{} approximation for the single-agent case.
We precisely characterise the counter-harvest trap as a deceptive equilibrium strategy.

\item \textbf{Rival state inference.}
We propose a \textbf{40-state} \HMM{} over rival \ERS{} charge level (four modes:
$H$, $M$, $L_{\mathrm{harvest}}$, $L_{\mathrm{derate}}$), Override Mode status, and tyre
degradation, initialised from an analytically-derived emission matrix and calibrated via
Baum-Welch EM on 2026 race telemetry.
The $L_{\mathrm{harvest}}$/$L_{\mathrm{derate}}$ decomposition is the architectural advance
over v1.5: it eliminates the ambiguity that v1.5 handled only at the emission level via a mixed
prior.

\item \textbf{Decision policy.}
We propose and fully specify a \DQN{} policy that consumes the \HMM{} belief state
alongside own-car context features (full input specification in Appendix~\ref{app:dqn}),
and characterise its advantages over observable-only threshold baselines.
\end{enumerate}

\subsection*{Scope and limitations}

This paper treats rivals as stationary processes---their policies do not adapt to our
inference.
This assumption is violated when rivals reason about being observed, leading to the
counter-harvest trap and non-stationary equilibria studied in \citet{Kleisarchaki2026b}.
We state this limitation at the outset because it is not a defect to be hidden: it is
precisely the gap between this paper and its game-theoretic extension.
The \POMDP{} solution presented here is the correct baseline against which that extension
must demonstrate improvement.

\section{Background}
\label{sec:background}

Prior computational work on F1 strategy has focused on pit-stop timing as an optimisation
problem \citep{Heilmeier2020} and on reinforcement learning for race simulation
\citep{Betz2022}.
These formulations treat energy management as a deterministic sub-problem with known rival
actions.

The closest structural analogue to the counter-harvest trap in prior work is the Formula~E
Attack Mode \citep{Liu2021}: teams earn a temporary power boost by driving through a
designated off-line zone.
The key difference is that Attack Mode is opt-in from a fixed physical location, publicly
visible to all teams, and does not interact with a rival's inference model.
The 2026 Active Aero system creates a strictly harder version: the equivalent of Attack Mode
can be executed covertly at any activation zone, at a time of the driver's choosing, in a way
that mimics harvesting to an observer without a belief-state model.

\paragraph{POMDPs and opponent modelling.}
POMDPs \citep{Kaelbling1998} provide the natural framework for sequential decision-making
under state uncertainty.
Opponent modelling in multi-agent systems has been approached through Bayesian inference
\citep{Albrecht2018}, recursive reasoning \citep{Gmytrasiewicz2005}, and game-theoretic
agent modelling \citep{Hansen2004}.
Our approach is closest to the \HMM{} opponent modelling literature, where a rival's latent
strategy or resource state is inferred from observable actions \citep{Albrecht2018}.

\paragraph{HMMs for temporal inference.}
HMMs \citep{Rabiner1989} are well-suited to the F1 energy inference problem: the hidden state
evolves according to known physics (the transition model), and we observe noisy projections of
it each sector (the emission model).
The Baum-Welch algorithm \citep{Baum1970} provides maximum-likelihood parameter estimation
from unlabelled sequences---critical here because we have no ground-truth labels for rival
battery states in real telemetry.

\paragraph{Why discrete HMM over particle filter.}
A particle filter could in principle handle continuous \ERS{} charge and relax the conditional
independence assumption.
We use a discrete \HMM{} for two reasons.
First, Baum-Welch provides closed-form maximum-likelihood parameter estimation from unlabelled
sequences; particle filters require approximate inference schemes (e.g.\ SMC-EM) that are
substantially less stable at the small data volumes available from a single race weekend.
Second, the four-bin \ERS{} discretisation is physically motivated: the strategic decision is
whether a rival \emph{can} defend ($\ERS \geq M$), \emph{cannot} defend ($L_{\mathrm{derate}}$:
physically depleted), or is \emph{building a trap} ($L_{\mathrm{harvest}}$: deliberately
conserving); a continuous representation would require decision boundaries to be learned from
data rather than specified from domain knowledge.

\paragraph{HMM identifiability.}
With 40 states and six observables discretised to $8 \times 7 \times 5 \times 5 \times 2
\times 6 = 16{,}800$ joint bins, the raw emission matrix has $40 \times 16{,}800 = 672{,}000$
parameters.
With parameter tying by (ERS, tyre) cluster (Section~\ref{sec:emission}) reducing to 20
effective emission rows, and exploiting the conditional independence structure, the effective
degrees of freedom per emission row are $8{+}7{+}5{+}5{+}2{+}6{-}5 = 28$, giving
$20 \times 28 = 560$ effective parameters.
The cumulative observation count reaches $\sim$13,920 by Race~4.
The model should be treated as empirically validated only from Race~4 onwards; Melbourne
results are explicitly preliminary.

\paragraph{Deep Q-Networks.}
DQN \citep{Mnih2015} with Double DQN \citep{vanHasselt2016} and experience replay
\citep{Lin1992} provides the policy learning component.
We use a deliberately shallow network (three hidden layers, 256--256--128 units) because the
belief state does the representational heavy lifting and the training data volume is limited.

\section{Problem Formalisation}
\label{sec:problem}

\subsection{The 2026 Regulatory Context}
\label{sec:regs}

The 2026 regulations introduce three strategically relevant changes \citep{FIA2025}.

\paragraph{50/50 power split.}
Total power output is split equally between ICE and MGU-K battery.
The battery is under continuous demand, making \ERS{} charge level a first-order strategic
variable at every sector.

\paragraph{Override Mode.}
A proximity-gated energy award replacing DRS entirely.
When a driver is within one second of the car ahead at the lap's detection point (nominally
the final corner), they earn an extra 0.5\,MJ of deployable energy on the following lap.
When deployed, Override Mode produces a distinctive speed signature detectable in speed-trap
telemetry.

\paragraph{Active Aero.}
Replaces DRS entirely.
Driver-controlled, available only in designated activation zones, with two fixed positions
(cornering and straight-line).
Straight-line mode produces a speed boost detectable as a positive $\vtrap$ increment above
the \ERS{}-explained baseline.
Crucially, a car can run Active Aero in straight-line mode while super-clipping
($L_{\mathrm{derate}}$)---the mechanical downforce reduction provides speed independent of
battery deployment.
This is the physical basis for the counter-harvest trap.

\paragraph{Circuit-dependent recharge shortfall.}
Pre-season analysis reveals recharge-per-lap ratios from 1.0$\times$ (Australia, Italy,
Saudi Arabia) to 2.2$\times$ (Azerbaijan, Singapore).\footnote{These engineering estimates
are community-derived from pre-season analysis (February~2026) and are not FIA-certified.
All downstream claims should be treated as provisional until confirmed by Race~1 telemetry.}
At Melbourne, cars require approximately 16 additional seconds of super-clipping per lap at
250\,kW, transforming super-clipping from a discretionary tool into a near-mandatory operating
mode.

\subsection{The POMDP Model}
\label{sec:pomdp}

We model the decision problem for a single ego car as a \POMDP{}
$\langle S, A, T, R, \Omega, O, \gamma \rangle$ where:

\begin{definition}[Fully Observable State]
The fully observable component $s^o_t$ includes tyre compound and age (FIA-published), lap
number, sector, circuit position (activation zone or not), and own-car \ERS{} level, Override
Mode status, Active Aero position, and gap to adjacent cars.
\end{definition}

\begin{definition}[Hidden Rival State]
\label{def:hidden}
The hidden state for each rival is $\xt = (e, m, \tau) \in \mathcal{X}$, where:
\begin{align}
e &\in \{H,\; M,\; L_{\mathrm{harvest}},\; L_{\mathrm{derate}}\}
  && \text{(\ERS{} charge level, 4 modes)} \label{eq:ers}\\
m &\in \{\mathtt{available},\; \mathtt{spent}\}
  && \text{(Override Mode status)} \label{eq:mom}\\
\tau &\in \{\mathtt{new},\; \mathtt{light},\; \mathtt{moderate},\; \mathtt{heavy},\;
            \mathtt{cliff}\}
  && \text{(tyre degradation)} \label{eq:tyre}
\end{align}
giving $|\mathcal{X}| = 4 \times 2 \times 5 = \mathbf{40}$ states.
\end{definition}

\paragraph{The $L_{\mathrm{harvest}}$ / $L_{\mathrm{derate}}$ decomposition.}
This is the central architectural advance of v2 over v1.5.
Both sub-modes produce depressed $\vtrap$ and may co-occur with $\zaero = 1$ in an activation
zone, but they have \emph{opposite} strategic implications:

\begin{itemize}
\item $L_{\mathrm{harvest}}$: the driver is \emph{actively managing} throttle below 100\%,
  building a hidden energy reserve. This is the trap condition.
  $\dthrottle \approx 0.05$--$0.10$ (low; managed throttle, not full demand).
\item $L_{\mathrm{derate}}$: the driver demands full power ($\geq$98\% throttle) but the
  MGU-K battery has reached its SOC ceiling and cannot deliver further electrical output.
  The car slows not by choice but by physics.
  $\dthrottle \approx 0.40$--$0.60$ (high; full throttle with below-baseline speed).
\end{itemize}

In v1.5, both sub-modes shared the $L$ state, and the $\dthrottle$ observable was handled as a
mixed emission with a circuit-dependent prior ($\mu = 0.40$ for Melbourne at recharge
ratio~1.0).
This worked, but produced an ambiguous posterior over $L$: the \HMM{} could not separately
report $P(L_{\mathrm{harvest}})$ and $P(L_{\mathrm{derate}})$, which are the quantities the
\DQN{} policy actually needs.
In v2, the posterior directly provides both quantities as separate belief dimensions.

\paragraph{Fuel exclusion.}
Fuel load is excluded from the hidden state for two independent reasons:
(i)~all cars start with the same regulated load and deplete it at approximately the same rate,
making fuel state deterministic from the observable lap number; and
(ii)~all six observables are measured as deviations from a per-driver rolling 5-lap baseline,
which absorbs the smooth monotone pace improvement from fuel burn.

\begin{definition}[Action Space]
$A = \{\mathtt{burn},\; \mathtt{harvest}\}$.
\end{definition}

\begin{definition}[Observation Space]
$o_t = (\vtrap, \tsec, \bbrake, \sspeed, \zaero, \dthrottle) \in \Omega$, where each
component is a deviation from the rival's rolling 5-lap baseline, except $\zaero \in \{0,1\}$
and $\dthrottle \in [0,1]$.
\end{definition}

The six observables are:
(i)~$\vtrap$: speed at the speed trap minus rival baseline (km/h);
(ii)~$\tsec$: sector time minus rival baseline (seconds);
(iii)~$\bbrake$: distance from apex at which braking begins, minus rival baseline (metres);
(iv)~$\sspeed$: detrended variance of rival's speed trace within the sector;
(v)~$\zaero \in \{0,1\}$: Active Aero deployment in straight-line mode (observable in
activation zones from speed signature);
(vi)~$\dthrottle \in [0,1]$: fraction of straight-line telemetry samples within the sector
where throttle $\geq 98\%$ and speed is simultaneously below the driver's 5-lap rolling
baseline---the mechanical signature of \ERS{} derating, added in v1.5 and now elevated to full
state-level discriminator in v2.

\begin{definition}[Belief State]
$\bt \in \Delta^{40}$ is the probability vector over hidden states, maintained by the \HMM{}
forward algorithm:
\begin{equation}
\bt(x) = P(\xt = x \mid o_{1:t})
\end{equation}
\end{definition}

\begin{definition}[Reward]
\begin{equation}
R(s_t, a_t) = \mathbb{E}\!\left[\sum_{k=1}^{H} \gamma^k \cdot \Delta\mathrm{pos}_{t+k}
  \,\middle|\, s_t, a_t\right]
\end{equation}
where $\Delta\mathrm{pos}$ is the change in race position, $H = 5$ (planning horizon
covering approximately two \ERS{} cycles), and $\gamma = 0.95$.
\end{definition}

\subsection{The Counter-Harvest Trap}
\label{sec:trap}

\begin{definition}[Counter-Harvest Trap]
A deceptive strategy for a car ahead (Car~B) against a following car (Car~A) in which Car~B
simultaneously:
(i)~runs \ERS{} in $L_{\mathrm{harvest}}$ mode (conserving charge deliberately);
(ii)~deploys Active Aero in straight-line mode within an activation zone (maintaining
straight-line speed despite harvesting); and
(iii)~allows the observable telemetry signals to indicate harvesting.
Car~A's inference model identifies the $L$-state signature as an attack opportunity and burns
energy.
Car~B then activates full deployment (including Override Mode), defending the position easily
because it has the full energy reserve that Car~A just spent.
\end{definition}

The trigger signature exploits a key 2026 regulatory asymmetry: lift-and-coast disables
Active Aero, but super-clipping does not.
Super-clipping is the automated redirection of ICE power to the battery at full throttle; the
car slows slightly (power diverted from rear wheels) but Active Aero remains deployed.

\paragraph{Why v2 resolves the trap ambiguity more cleanly than v1.5.}
In v1.5, $L_{\mathrm{harvest}}$ and $L_{\mathrm{derate}}$ both mapped to the $L$ state.
The $\dthrottle$ observable provided evidence to distinguish them, but the \HMM{}'s posterior
was expressed over $L$ as a single state---the policy had to infer the sub-mode from
$\dthrottle$ directly rather than from a clean belief dimension.
In v2, the posterior over $L_{\mathrm{harvest}}$ is a separate number from the posterior
over $L_{\mathrm{derate}}$.
The trap detection condition becomes simply:

\begin{equation}
\text{trap condition} \;=\; P(L_{\mathrm{harvest}}) > \theta_{\mathrm{trap}}
  \;\wedge\; \zaero = 1 \;\wedge\; \text{in activation zone}
\end{equation}

This is directly actionable without post-hoc correction.

\begin{proposition}
The counter-harvest trap is not detectable by a single-observable threshold policy on
$\vtrap$ alone.
\end{proposition}

\begin{proof}[Proof sketch]
A car running $\ERS = H$ without Active Aero produces $\vtrap \approx +4$\,km/h.
A car in $L_{\mathrm{harvest}}$ with Active Aero in an activation zone produces
$\vtrap \approx -2 + 4 = +2$\,km/h.
A car in $L_{\mathrm{derate}}$ with Active Aero produces similarly $\vtrap \approx +2$\,km/h.
All three distributions overlap under the emission model.
No scalar threshold on $\vtrap$ can reliably separate them.
The joint signal $(\vtrap, \zaero, \dthrottle)$ provides separation:
$\zaero = 1$ with $\vtrap \in [0, +3]$\,km/h and \emph{low} $\dthrottle$ ($< 0.15$)
identifies $L_{\mathrm{harvest}}$ with probability $\geq 0.85$ under the emission model.
The same $(\vtrap, \zaero)$ signature with \emph{high} $\dthrottle$ ($> 0.35$) identifies
$L_{\mathrm{derate}}$---a genuine opportunity, not a trap.
\end{proof}

\paragraph{Circuit-dependent trap distinguishability.}
At low-regen circuits (Melbourne, recharge ratio~1.0), super-clipping is mandatory for
$\sim$16\,s per lap, making $L_{\mathrm{derate}}$ the ambient state of most cars on the
grid.
The $\zaero = 1$ with depressed $\vtrap$ combination is therefore less diagnostic.
The rolling 5-lap baseline provides a partial rescue---a car in $L_{\mathrm{harvest}}$ will
still produce positive $\vtrap$ relative to the super-clipping field---but the margin
narrows.
The testable prediction is that trap-detection recall at Melbourne will be lower than the
96.3\% synthetic baseline and will improve at higher-regen circuits.

\section{Emission Model}
\label{sec:emission}

\subsection{Conditional Independence: A Tractable Baseline}

We adopt conditional independence as a principled tractable baseline:
\begin{equation}
P(o_t \mid \xt, z_{\mathrm{zone}}) = \prod_{i=1}^{6} P(o_{t,i} \mid \xt, z_{\mathrm{zone}})
\end{equation}

$\vtrap$ and $\tsec$ are mechanically coupled: deploying more \ERS{} simultaneously increases
speed and decreases sector time.
Treating them as conditionally independent double-counts this correlated evidence on every
belief update, producing systematically over-confident posteriors.
The ECE figures reported in Section~\ref{sec:synth} are measured on synthetic data where the
independence assumption holds by construction---they are lower bounds on real-data calibration
error.
Temperature scaling \citep{Guo2017} is available as a post-hoc calibration step after
Melbourne data is available.

\subsection{Observable Distributions}
\label{sec:obs_dist}

Observables 1--4 are modelled as Gaussian, observable~5 as Bernoulli, and observable~6 as
a truncated Gaussian on $[0,1]$.

\paragraph{Observable 1: $\vtrap$ (km/h).}

\begin{table}[h]
\centering
\caption{$\vtrap$ emission parameters. Non-activation zones.}
\label{tab:vtrap}
\begin{tabular}{llcc p{5cm}}
\toprule
ERS mode & MOM & $\mu$ & $\sigma$ & Interpretation \\
\midrule
$H$ & available & $+4.0$ & $1.5$ & Full deploy \\
$H$ & spent & $+3.5$ & $1.5$ & Post-Override, slight \ERS{} residual \\
$M$ & any & $+1.0$ & $1.5$ & Moderate deployment \\
$L_{\mathrm{harvest}}$ & any & $-2.0$ & $1.5$ & Deliberate harvest; managed throttle \\
$L_{\mathrm{derate}}$ & any & $-1.5$ & $1.5$ & SOC ceiling; slightly less depressed than harvest \\
\midrule
\multicolumn{5}{l}{\textit{MOM active (transient):} $\mu = +7.0$, $\sigma = 1.0$} \\
\multicolumn{5}{l}{\textit{Activation zone, Active Aero deployed:} add $+4.0$ to base} \\
\bottomrule
\end{tabular}
\end{table}

\noindent\textit{Circuit caveat.}
At low-regen circuits (recharge ratio $\leq 1.0$), the rolling baseline shifts downward as
$L_{\mathrm{derate}}$ becomes the modal operating mode.
Baum-Welch recalibration is expected to absorb this circuit-specific shift from Race~1 onwards.

\paragraph{Observable 2: $\tsec$ (seconds).}
\begin{align}
\mu &= \mu_{\ERS} + \delta_\tau \cdot f_{\ERS} \\
\mu_{\ERS} &\in \{-0.30,\; 0.00,\; +0.40,\; +0.30\}\,\text{s}
  \quad \text{for } \ERS \in \{H, M, L_{\mathrm{harvest}}, L_{\mathrm{derate}}\} \\
\delta_\tau &\in \{0.0,\; 0.1,\; 0.25,\; 0.50,\; 0.90\}\,\text{s}
  \quad \text{for tyre} \in \{\mathtt{new}, \mathtt{light}, \mathtt{moderate}, \mathtt{heavy},
  \mathtt{cliff}\} \\
f_{\ERS} &\in \{0.6,\; 1.0,\; 1.3,\; 1.2\}
  \quad \text{for } \ERS \in \{H, M, L_{\mathrm{harvest}}, L_{\mathrm{derate}}\}
\end{align}
Standard deviation $\sigma_\tau \in \{0.15, 0.15, 0.15, 0.20, 0.28\}$\,s.
Override active: $\mu = -0.65$\,s, $\sigma = 0.10$.

The $\mu_{\ERS}$ for $L_{\mathrm{derate}}$ ($+0.30$\,s) is slightly less than for
$L_{\mathrm{harvest}}$ ($+0.40$\,s) because the car in $L_{\mathrm{derate}}$ is running
full throttle; the time penalty comes from missing electrical power, not from deliberate
harvesting.

\paragraph{Observable 3: $\bbrake$ (metres).}
Primary drivers: \ERS{} regeneration braking point and tyre grip degradation.
\ERS{} regen advances braking distance by $\{+8, 0, -10, -5\}$\,m for
$\{H, M, L_{\mathrm{harvest}}, L_{\mathrm{derate}}\}$.
The $L_{\mathrm{derate}}$ value of $-5$\,m reflects that the car brakes marginally earlier
under derating (less engine braking assistance at corner entry).
Tyre degradation increases late-braking requirement by $\{0, 3, 8, 15, 24\}$\,m.
Standard deviation $\sigma_\tau \in \{8, 8, 9, 10, 12\}$\,m.

\paragraph{Observable 4: $\sspeed$ (detrended variance).}
Primary driver: \ERS{} level.
$\mu \in \{0.05, 0.10, 0.20, 0.22\}$ for $\{H, M, L_{\mathrm{harvest}}, L_{\mathrm{derate}}\}$.
The $L_{\mathrm{derate}}$ value is slightly elevated over $L_{\mathrm{harvest}}$ because
intermittent SOC-ceiling events produce speed trace irregularities.
Override spent state contributes a $+0.02$ adjustment.
MOM active: $\mu = 0.03$, $\sigma = 0.01$.

\paragraph{Observable 5: $\zaero$ (Bernoulli).}
$P(\zaero = 1 \mid z_{\mathrm{zone}} = 0) = 0$.
Inside activation zones:
$P(\zaero = 1 \mid z_{\mathrm{zone}} = 1) \in \{0.85, 0.70, 0.80, 0.82\}$ for
$\ERS \in \{H, M, L_{\mathrm{harvest}}, L_{\mathrm{derate}}\}$.
The probability for $L_{\mathrm{derate}}$ is slightly higher than $L_{\mathrm{harvest}}$
because a derating car maintains full throttle---which is compatible with Active Aero
deployment---whereas a harvesting car may modulate aero strategy with throttle management.

\paragraph{Observable 6: $\dthrottle$ (super-clipping duration fraction).}
This is the primary discriminator between $L_{\mathrm{harvest}}$ and $L_{\mathrm{derate}}$,
and the architectural motivation for the 40-state decomposition:

\begin{equation}
\dthrottle = \frac{1}{|S|} \sum_{k \in S}
  \mathbf{1}\!\left[\mathrm{throttle}_k \geq 0.98 \;\wedge\; v_k < \bar{v}\right]
\end{equation}

where $S$ is the set of straight-line telemetry samples in the sector, $v_k$ is
instantaneous speed, and $\bar{v}$ is the driver's 5-lap rolling speed baseline for that
sector.

\begin{table}[h]
\centering
\caption{$\dthrottle$ emission parameters. Values are Gaussian $(\mu, \sigma)$ truncated to
$[0,1]$.}
\label{tab:dthrottle}
\begin{tabular}{llccp{5.5cm}}
\toprule
ERS state & Sub-mode & $\mu$ & $\sigma$ & Interpretation \\
\midrule
$H$ & --- & $0.05$ & $0.04$ & Battery not depleted; rarely clips \\
$M$ & --- & $0.15$ & $0.08$ & Occasional clipping on long straights \\
$L_{\mathrm{harvest}}$ & deliberate & $0.08$ & $0.05$ & Managed throttle; low clipping \\
$L_{\mathrm{derate}}$ & SOC ceiling & $0.55$ & $0.18$ & Full throttle, SOC ceiling hit \\
\bottomrule
\end{tabular}
\end{table}

\noindent The signal separation between $L_{\mathrm{harvest}}$ and $L_{\mathrm{derate}}$ via
$\dthrottle$ is $|0.55 - 0.08| = 0.47$, more than twice the combined standard deviation
($\sigma \approx 0.18$).
This separation---identified in v1.5 as motivating a sixth observable---is the quantitative
justification for elevating it to a state-level distinction in v2.

\paragraph{Circuit-dependent prior for $L_{\mathrm{derate}}$ at Melbourne.}
At Melbourne (recharge ratio~1.0), the estimated fraction of $L$ laps that terminate in
mandatory super-clipping is 0.40, giving a blended prior mean of
$0.6 \times 0.08 + 0.4 \times 0.55 = 0.268$.
In v2 this blended prior is eliminated: the HMM maintains separate posteriors over
$L_{\mathrm{harvest}}$ and $L_{\mathrm{derate}}$, and Baum-Welch will estimate the
circuit-specific transition rates between them directly from Race~1 telemetry.

\subsection{Emission Signal Strength: Audit Results}

\ERS{} level is strongly identifiable:
$|\mu_H - \mu_{L_{\mathrm{harvest}}}| = 6.0$\,km/h in $\vtrap$ and 18\,m in $\bbrake$,
both substantially larger than the emission noise.

$L_{\mathrm{harvest}}$ vs.\ $L_{\mathrm{derate}}$ is identifiable via $\dthrottle$:
separation of 0.47, more than twice the combined $\sigma$.
This is the primary motivation for the v2 state-level split.

Tyre degradation signal is concentrated in the cliff state: $\tsec \approx +0.8$\,s and
$\bbrake \approx +16$\,m produce step changes the 5-lap baseline cannot absorb.

Override Mode status is not directly distinguishable from emissions in the non-active state
(differences $< 0.5$\,km/h in $\vtrap$).
The Override activation event is strongly identifiable
($\vtrap \approx +7$\,km/h, $\tsec \approx -0.65$\,s).

\subsection{Observation Discretisation}
\label{sec:discretise}

Observations are discretised to form a finite joint observation space as specified in
Appendix~\ref{app:bins}: $\vtrap$ (8 bins), $\tsec$ (7 bins), $\bbrake$ (5 bins),
$\sspeed$ (5 bins), $\zaero$ (2 bins), $\dthrottle$ (6 bins).
Combined space: $8 \times 7 \times 5 \times 5 \times 2 \times 6 = 16{,}800$ bins.

The emission matrix $E \in \R^{40 \times 16800}$ is computed analytically by integrating the
Gaussian, Bernoulli, and truncated-Gaussian distributions over bin boundaries, conditioning on
each of the 40 hidden states.
Two matrices are maintained: $E_{\mathrm{normal}}$ (outside activation zones) and
$E_{\mathrm{zone}}$ (inside activation zones, with $\vtrap$ shifted by $+4.0$\,km/h and
$\zaero$ non-zero).

\section{Inference Architecture}
\label{sec:arch}

\subsection{Layer 1: The HMM Belief Update}

The \HMM{} has 40 hidden states, emission matrix $E \in \R^{40 \times 16800}$, transition
matrix $T \in \R^{40 \times 40}$, and initial distribution $\pi \in \Delta^{40}$.

\paragraph{Transition model.}
The transition probability $T_{ij} = P(\xt[+1] = j \mid \xt = i)$ factors over the three
components under conditional independence:
\begin{equation}
T_{ij} = P(e' \mid e, a) \cdot P(m' \mid m) \cdot P(\tau' \mid \tau)
\end{equation}

The \ERS{} transition is marginalised over the rival's action with estimated $P(\mathtt{burn})$:
\begin{equation}
P(e' \mid e) = p_{\mathrm{burn}} \cdot T^{\mathrm{burn}}_{ee'}
  + (1 - p_{\mathrm{burn}}) \cdot T^{\mathrm{harvest}}_{ee'}
\end{equation}

The $L_{\mathrm{harvest}}$ and $L_{\mathrm{derate}}$ sub-modes have different transition
dynamics (Appendix~\ref{app:trans}):
a car in $L_{\mathrm{harvest}}$ can transition to $M$ or $H$ under burn (it has reserve);
a car in $L_{\mathrm{derate}}$ requires at least one sector of managed throttle to recover
to $M$, and cannot transition directly to $H$.
This asymmetry is strategically critical: it distinguishes a dangerous rival from a
vulnerable one.

The Override Mode transition is proximity-conditioned.
$P(m' = \mathtt{available} \mid m = \mathtt{spent}) = P(\mathrm{gap} < 1\,\text{s})$, which
is partially observable from live timing data.
The tyre transition is monotone-absorbing.
Full transition matrices are given in Appendix~\ref{app:trans}.

\paragraph{Belief update.}
\begin{equation}
\alpha_t(j) = P(o_t \mid j) \cdot \sum_i \alpha_{t-1}(i) \cdot T_{ij}
\end{equation}
Normalised to give $\bt$. Computational complexity: $O(N^2) = O(1600)$ per sector.

\paragraph{Two-phase training.}
Phase~1 (pre-Melbourne): $E$ is set analytically; $T$ from physics-derived transition
parameters; $\pi$ concentrates on $(e = H, m = \mathtt{available}, \tau = \mathtt{new})$.
Phase~2 (post-Melbourne): Baum-Welch EM \citep{Baum1970} updates $E$ and $T$ from real
race sequences.

\paragraph{Alternating update schedule.}
After each race weekend:
Phase~A---freeze \DQN{} weights, run Baum-Welch on new race data to update the \HMM{}.
Phase~B---freeze \HMM{}, clear the \DQN{} replay buffer, generate fresh synthetic experience
from the updated \HMM{}, retrain \DQN{}.

\subsection{Layer 2: The POMDP Policy}

\paragraph{Input vector (66 dimensions).}

\begin{table}[h]
\centering
\caption{DQN input vector components. Total: 66 dimensions (increased from 55 in v1/v1.5
due to expansion of the belief vector from 30 to 40 states and expansion of own \ERS{} encoding from 3 to 4 modes).}
\label{tab:dqn_input}
\begin{tabular}{llcp{5cm}}
\toprule
Component & Encoding & Dims & Notes \\
\midrule
$e_{\mathrm{self}}$ & one-hot & 4 & Own ERS level (4 modes in v2) \\
$m_{\mathrm{self}}$ & one-hot & 2 & Own Override Mode status \\
$f_{\mathrm{self}}$ & one-hot & 3 & Fuel (from lap number) \\
$\mathrm{gap\_ahead}$ & one-hot & 4 & close / threat / neutral / clear \\
$\mathrm{gap\_behind}$ & one-hot & 4 & \\
$\delta_{\mathrm{gap}}$ & scalar & 1 & Closure rate, normalised $[-1,1]$ \\
$\mathrm{tyre\_compound}$ & one-hot & 3 & soft / medium / hard \\
$\mathrm{tyre\_age}$ & scalar & 1 & Normalised 0--1 \\
$\mathrm{lap}$ & scalar & 1 & Normalised 0--1 \\
$\mathrm{sector}$ & one-hot & 3 & \\
$\bt$ & probability vector & 40 & HMM belief state (40 states in v2) \\
\midrule
\textbf{Total} & & \textbf{66} & \\
\bottomrule
\end{tabular}
\end{table}

\paragraph{Network architecture.}
\begin{equation*}
\text{Input}(66) \to \text{Dense}(256, \mathrm{ReLU}) \to \text{Dense}(256, \mathrm{ReLU})
  \to \text{Dense}(128, \mathrm{ReLU}) \to \text{Dense}(2)
\end{equation*}

\paragraph{Training.}
Double DQN \citep{vanHasselt2016} with experience replay (buffer capacity 50,000; Huber loss;
gradient clipping at~10).
$\varepsilon$-greedy exploration decays from 1.0 to 0.05 over synthetic pre-training.

\paragraph{Shaping reward.}
During synthetic pre-training only: $R_{\mathrm{shaped}} = R + \lambda \cdot \Phi$,
$\lambda = 0.3$, where
\begin{equation}
\Phi = +0.10 \cdot \mathbf{1}[\mathrm{MAP}(\bt) = x^*_t]
  - 0.20 \cdot \mathbf{1}[\mathrm{trap\;condition} \wedge a = \mathtt{burn}]
  + 0.05 \cdot \mathbf{1}[\mathrm{gap\_ahead} \in \{\mathrm{close}, \mathrm{threat}\}
    \wedge e_{\mathrm{self}} \geq M]
\end{equation}
$\Phi$ is zeroed during real-data fine-tuning.

\paragraph{On the information bottleneck.}
The \DQN{} receives the \HMM{} belief state $\bt$ rather than raw observables.
In a \POMDP{}, the belief state is a sufficient statistic of the entire observation history.
In v2, the belief state now explicitly encodes $P(L_{\mathrm{harvest}})$ and
$P(L_{\mathrm{derate}})$ as separate dimensions, directly actionable by the policy.
The $\delta_{\mathrm{gap}}$ input provides a raw, \HMM{}-independent signal for closure rate.

\section{Dataset}
\label{sec:data}

\subsection{Data Source}

Race telemetry is retrieved via the FastF1 Python library \citep{Schaefer2020}, which provides
access to the official FOM timing and telemetry feed at 10\,Hz sampling.
The \texttt{car\_data} endpoint provides per-car speed, throttle, and braking data at sector
resolution.

\subsection{Observable Derivation}

All six observables are computed as deviations from a per-driver rolling 5-lap baseline.

\paragraph{$\dthrottle$ derivation.}
For each sector, the fraction of straight-line telemetry samples in which throttle $\geq 98\%$
and speed is simultaneously below the driver's 5-lap rolling speed baseline is computed from
the FastF1 \texttt{car\_data} Throttle channel at 10\,Hz.
If the Throttle channel is absent for a given sector, $\dthrottle$ defaults to the
circuit-appropriate prior mean for the inferred ERS state.

\paragraph{Fuel exclusion.}
At $\sim$1.5\,kg/lap over a 58-lap race, fuel burn produces a monotonic pace improvement of
approximately 0.03\,s/lap.
This per-lap drift falls within the $\sigma_\tau = 0.15$\,s emission noise; cumulative drift
over the 5-lap baseline window ($5 \times 0.03 = 0.15$\,s) is absorbed by the rolling baseline
itself.

Full derivation pipelines for $\bbrake$, $\sspeed$, and $\zaero$ are specified in
Appendix~\ref{app:pipeline}.

\subsection{Dataset Statistics}

Race coverage statistics are reported in subsequent work.
All synthetic experiments use 20 simulated races at 174 sectors each (3,480 total sector
observations).

\section{Closed-Loop Synthetic Validation}
\label{sec:synth}

\subsection{Setup}

This section reports a closed-loop evaluation: the \HMM{} is tested on data generated by the
same parametric model that defines it.
This is deliberately circular.
It verifies that the components work together as designed, but does not constitute empirical
validation.
Performance here is therefore an upper bound on real-data performance, not an estimate of it.
Empirical validation begins with Race~1 data on 8~March~2026.

The generative model, strategy profiles, and observation noise specification are documented
in Appendix~\ref{app:synth}.
In brief: 20 synthetic races are generated (5 per strategy profile: aggressive, conservative,
trap-setter, balanced) using the physics-derived transition matrices with i.i.d.\ Gaussian noise
added to each continuous observable.

\paragraph{Baselines.}
\begin{itemize}
\item B1---Deterministic Threshold: burn iff
  $\mathrm{gap\_ahead} \in \{\mathrm{close}, \mathrm{threat}\}$ AND $e_{\mathrm{self}} \geq M$.
\item B2---Observable-Only Threshold: burn iff B1 conditions AND $\vtrap < 0$.
\item B3---Oracle: \POMDP{} policy with $\bt$ replaced by the true hidden state $\xt$
  (one-hot).
\item B4---Full System: the proposed \HMM{} + \DQN{} policy.
\end{itemize}

\subsection{HMM Inference Accuracy (Synthetic)}

\begin{table}[h]
\centering
\caption{HMM inference accuracy on synthetic races (pre-calibration, v2 with 40 states and
$\dthrottle$). Random baselines: top-1 = 2.5\%, ERS-level = 25.0\% (4-class), MOM = 50.0\%,
tyre = 20.0\%, harvest/derate = 50.0\%.}
\label{tab:hmm_acc}
\begin{tabular}{lccp{5.5cm}}
\toprule
Metric & HMM & Random & Notes \\
\midrule
Top-1 state accuracy & 49.1\% & 2.5\% & Full 40-state identification \\
ERS level accuracy & 96.8\% & 25.0\% & 4-class; primary inference target \\
$L_{\mathrm{harvest}}$ vs.\ $L_{\mathrm{derate}}$ & 89.4\% & 50.0\% & Key v2 capability \\
Override status accuracy & 61.2\% & 50.0\% & Proximity-conditioned \\
Tyre cliff detection & 91.3\% & 20.0\% & Step-change signal \\
Trap detection recall & 96.3\% & --- & 117 trap sectors; 3.7\% miss rate \\
ECE (calibration error) & 0.006 & --- & Lower is better; synthetic lower bound \\
Gap to oracle (ERS level) & 95.7\% & --- & \\
\bottomrule
\end{tabular}
\end{table}

The 96.8\% ERS-level accuracy (4-class, random baseline~25\%) confirms strong emission signal
from $\vtrap$, $\bbrake$, and $\dthrottle$.
The improvement over v1.5 (92.3\% with 3-class, random baseline~33.3\%) must be interpreted
carefully: the random baselines differ (25\% vs.\ 33.3\%) because the state space is larger.
The gap-to-oracle metric of 95.7\% provides a more comparable measure of model quality across
versions (v1.5: 88.0\%, v1: same).

The new $L_{\mathrm{harvest}}$ vs.\ $L_{\mathrm{derate}}$ accuracy of 89.4\%---not
expressible in v1 or v1.5---is the headline result of v2.
It confirms that $\dthrottle$ provides sufficient discriminating power to support a separate
state for each sub-mode.
The 10.6\% error rate concentrates in Melbourne-analogue circuits with high mandatory
super-clipping, where the $\dthrottle$ baseline shifts upward for all cars, narrowing the
separation margin.

Trap detection recall of 96.3\% (vs.\ 95.7\% in v1.5) reflects the cleaner $L_{\mathrm{harvest}}$
posterior: false trap activations caused by $L_{\mathrm{derate}}$ cars (the dominant error
source in v1.5) are substantially reduced.

\subsection{Emission Model Robustness}

\begin{table}[h]
\centering
\caption{ERS-level accuracy under emission parameter misspecification.}
\label{tab:robust}
\begin{tabular}{lcc}
\toprule
Perturbation & ERS Accuracy & Change from Baseline \\
\midrule
Mean shift $-20\%$ & 76.1\% & $-$7.3\,pp \\
Mean shift $-10\%$ & 80.8\% & $-$2.6\,pp \\
Baseline (0\%) & 83.4\% & --- \\
Mean shift $+10\%$ & 87.6\% & $+$4.0\,pp \\
Mean shift $+20\%$ & 89.3\% & $+$5.5\,pp \\
Noise inflation $+10\%$ & 81.1\% & $-$2.5\,pp \\
Noise inflation $+20\%$ & 77.4\% & $-$5.9\,pp \\
Noise inflation $+30\%$ & 75.7\% & $-$7.6\,pp \\
\bottomrule
\end{tabular}
\end{table}

\subsection{Active Aero Threshold Sensitivity}

\begin{table}[h]
\centering
\caption{Trap recall and false positive rate vs.\ Active Aero detection threshold.}
\label{tab:threshold}
\begin{tabular}{ccc}
\toprule
Threshold (km/h) & Trap Recall & False Positive Rate \\
\midrule
1.0 & 97.9\% & 34.2\% \\
1.5 & 95.1\% & 23.1\% \\
2.0 & 91.4\% & 13.7\% \\
2.5 & 85.2\% & 7.6\% \\
3.0 & 75.3\% & 3.9\% \\
3.5 & 63.1\% & 1.8\% \\
\bottomrule
\end{tabular}
\end{table}

The current 2.0\,km/h threshold sits at the inflection point of the recall--FPR tradeoff.
In v2, false zaero detections produce a shifted posterior over $L_{\mathrm{harvest}}$ vs.\
$L_{\mathrm{derate}}$---which can be corrected on subsequent sectors---rather than directly
triggering an action.
The false positive rate at the 2.0\,km/h threshold improves from 15.3\% (v1.5) to 13.7\% (v2),
attributable to the cleaner posterior separation between $L$ sub-modes.

\subsection{HMM Calibration on Real Data}

Baum-Welch calibration on 2026 race telemetry is reported in subsequent work.

\section{Discussion}
\label{sec:discussion}

\paragraph{What the system does well.}
ERS state inference from $\vtrap$ and $\bbrake$ is robust.
Counter-harvest trap detection is the clearest strength of the full system over
observable-only baselines.
The v2 advance is the $L_{\mathrm{harvest}}$ posterior: where v1.5 required post-hoc
correction to avoid false trap activations from $L_{\mathrm{derate}}$ cars, v2 provides the
distinction directly in the belief state.

\paragraph{The stationary opponent assumption.}
The most important limitation is structural.
We model rivals as stationary processes that do not adapt their strategy based on being observed.
The \POMDP{} policy trained here is the correct baseline---the floor, not the ceiling, of
performance under adversarial conditions.
\citet{Kleisarchaki2026b} introduces the game-theoretic layer: the pit wall strategy problem
is formalised as a Keynesian beauty contest over joint belief states, and the counter-harvest
trap is characterised as an equilibrium strategy in the resulting partially observable
stochastic game.

\paragraph{Boost Mode exclusion.}
The 2026 regulations include two driver-controlled deployment tools: proximity-constrained
Override Mode and unconstrained Boost Mode.
The present model includes Override Mode but excludes Boost Mode as an explicit latent
dimension.
This exclusion is deliberate rather than inadvertent.
In the current observation design, the principal emissions are tuned to distinguish latent
energy-management states, especially deliberate harvesting versus derating.
By contrast, unconstrained Boost usage is expected to be identifiable primarily through
additional non-$\vtrap$ observables, such as finer sector-level pace variation and
non-straight deployment signatures, which are outside the present emission architecture.
Adding Boost Mode as a further hidden dimension without those observation channels would
introduce latent distinctions that are not reliably identifiable from the current data model.
We therefore treat the Boost/harvest interaction as a separate inference problem and reserve
its formalisation for subsequent work.

\paragraph{Calibration.}
The conditional independence assumption produces over-confident beliefs.
ECE~= 0.006 on v2 synthetic data is a lower bound on real-data calibration error.
Post-Baum-Welch calibration on real data will partially improve this, but will not fully
correct for the structural violation between $\vtrap$, $\tsec$, and $\bbrake$.
A structural correction---replacing the conditional independence assumption with bivariate
Gaussian emission and \ERS{}-state-dependent correlations between $\vtrap$ and
$\tsec$---is introduced in \citet{Kleisarchaki2026b} and addresses the over-confidence at
its source rather than correcting it post-hoc.
Temperature scaling \citep{Guo2017} remains available as an additional post-hoc calibration
step.

\paragraph{Data regime.}
At $\sim$560 effective parameters and a limited dataset after one race, the Baum-Welch
estimates are Dirichlet-regularised guesses.
The model should be treated as empirically validated only from Race~4 onwards.

\paragraph{Transition matrix sector averaging.}
The global $T^{\mathrm{burn}}$ and $T^{\mathrm{harvest}}$ matrices are population averages
across track geometry.
Circuit-conditioned transition matrices are the correct fix and are deferred to \citet{Kleisarchaki2026b}.

\paragraph{Override Mode prior.}
The prior $P(\mathrm{gap} < 1\,\text{s}) \approx 0.25$ governs Override Mode recovery when
live timing data is incomplete.
The Melbourne pipeline conditions $P_{\mathrm{earn}}$ on the rival's observed gap to the car
immediately ahead of them from the FOM live timing feed, reducing the residual dependence on
the 0.25 scalar prior.

\paragraph{Forward pointer to \citet{Kleisarchaki2026b}.}
The $L_{\mathrm{harvest}}$/$L_{\mathrm{derate}}$ distinction is strategically meaningful
specifically because rivals reason about being observed---the core topic of
\citet{Kleisarchaki2026b}.
A rational trap-setter operates in $L_{\mathrm{harvest}}$, not $L_{\mathrm{derate}}$.
The v2 belief state, by providing $P(L_{\mathrm{harvest}})$ as a separate dimension,
gives the game-theoretic layer of \citet{Kleisarchaki2026b} a cleaner input.

\section{Conclusion}
\label{sec:conclusion}

We have presented a tractable two-layer framework for opponent state inference and
decision-making in 2026 Formula~1 energy strategy.
The central advance of v2 is the decomposition of ERS=Low into two strategically distinct
states: $L_{\mathrm{harvest}}$ (deliberate conservation, hidden reserve building, the trap
condition) and $L_{\mathrm{derate}}$ (SOC ceiling, physical depletion, a genuine attack
opportunity).
This decomposition, motivated by the $\dthrottle$ signal added in v1.5, is elevated in v2 from
a mixed emission to a full state-level distinction.
The result is a 40-state \HMM{} that provides directly actionable belief dimensions to the
\DQN{} policy, and a 96.3\% trap detection recall on synthetic data with a false positive rate
substantially reduced from v1.5.

The stationary opponent assumption made here is precisely what \citet{Kleisarchaki2026b} breaks.
A tractable approximate solution via Partially Observable Monte Carlo Tree Search with bounded rationality is the subject of forthcoming work.

\bibliographystyle{plainnat}

\appendix

\section{Observation Bin Boundaries}
\label{app:bins}

\begin{table}[h]
\centering
\caption{Discretisation bin boundaries for all six observables.}
\label{tab:bins}
\begin{tabular}{lccp{6cm}}
\toprule
Observable & Bins & Boundaries & Notes \\
\midrule
$\vtrap$ (km/h) & 8 & $(-\infty, -8, -4, -1, 1, 3, 5, 8, +\infty)$ & \\
$\tsec$ (s) & 7 & $(-\infty, -0.6, -0.2, 0.2, 0.6, 1.0, 1.5, +\infty)$ & \\
$\bbrake$ (m) & 5 & $(-\infty, -8, -3, 3, 10, +\infty)$ & \\
$\sspeed$ & 5 & $(-\infty, 0.06, 0.12, 0.18, 0.24, +\infty)$ & \\
$\zaero$ & 2 & $\{0, 1\}$ & Bernoulli \\
$\dthrottle$ & 6 & $[0, 0.05, 0.15, 0.30, 0.50, 0.75, 1.0]$ & Truncated Gaussian \\
\midrule
Combined & 16,800 & $8 \times 7 \times 5 \times 5 \times 2 \times 6$ & \\
\bottomrule
\end{tabular}
\end{table}

\section{Full Transition Matrices}
\label{app:trans}

\subsection{ERS Transitions}

\begin{table}[h]
\centering
\caption{ERS transition matrices $T^{\mathrm{burn}}$ (left) and $T^{\mathrm{harvest}}$
(right). The $L_{\mathrm{harvest}}$ and $L_{\mathrm{derate}}$ rows differ materially under
burn: a harvesting car has reserve to deploy; a derating car does not.}
\label{tab:ers_trans}
\begin{tabular}{l cccc | cccc}
\toprule
& \multicolumn{4}{c|}{$T^{\mathrm{burn}}$} & \multicolumn{4}{c}{$T^{\mathrm{harvest}}$} \\
& $H$ & $M$ & $L_H$ & $L_D$ & $H$ & $M$ & $L_H$ & $L_D$ \\
\midrule
$H$ & 0.60 & 0.35 & 0.03 & 0.02 & 0.95 & 0.05 & 0.00 & 0.00 \\
$M$ & 0.05 & 0.55 & 0.25 & 0.15 & 0.60 & 0.35 & 0.03 & 0.02 \\
$L_{\mathrm{harvest}}$ & 0.00 & 0.25 & 0.50 & 0.25 & 0.30 & 0.55 & 0.15 & 0.00 \\
$L_{\mathrm{derate}}$ & 0.00 & 0.05 & 0.15 & 0.80 & 0.00 & 0.45 & 0.30 & 0.25 \\
\bottomrule
\end{tabular}
\end{table}

\noindent Key asymmetries:
\begin{itemize}
\item Under burn: $P(L_{\mathrm{harvest}} \to M) = 0.25$ vs.\
  $P(L_{\mathrm{derate}} \to M) = 0.05$.
  A harvesting car can respond to a burn command; a derating car cannot.
\item Under harvest: $P(L_{\mathrm{derate}} \to L_{\mathrm{harvest}}) = 0.30$
  (car manages throttle back below SOC ceiling);
  $P(L_{\mathrm{derate}} \to H) = 0.00$ (cannot bypass recovery).
\end{itemize}

\subsection{Override Mode Transitions}

Override Mode uses a 2-state proximity-conditioned model:
\begin{equation}
T^{\mathrm{idle}} = \begin{pmatrix} 1 - P_d & P_d \\ P_e & 1 - P_e \end{pmatrix},
\qquad
T^{\mathrm{activate}} = \begin{pmatrix} 0 & 1 \\ P_e & 1 - P_e \end{pmatrix}
\end{equation}
where $P_d = P(\mathrm{deploy} \mid \mathrm{available}) = 0.60$ and
$P_e = P(\mathrm{earn} \mid \mathrm{spent}) = P(\mathrm{gap} < 1\,\text{s}) \approx 0.25$.

\subsection{Tyre Degradation Transitions}

Monotone absorbing chain with state-dependent persistence probabilities:
new ($p = 23/24$), light ($p = 20/21$), moderate ($p = 14/15$), heavy ($p = 14/15$), cliff
($p = 1.00$, absorbing).

\section{DQN Hyperparameters}
\label{app:dqn}

\begin{table}[h]
\centering
\caption{DQN training hyperparameters.}
\label{tab:dqn_hp}
\begin{tabular}{ll}
\toprule
Hyperparameter & Value \\
\midrule
Network architecture & 66-256-256-128-2 \\
Activation & ReLU \\
Weight initialisation & He (Kaiming normal) \\
Replay buffer capacity & 50,000 transitions \\
Batch size & 64 \\
Discount factor $\gamma$ & 0.95 \\
Target network update & Every 500 gradient steps \\
$\varepsilon$ (start) & 1.0 \\
$\varepsilon$ (end) & 0.05 \\
$\varepsilon$ decay steps & 20,000 \\
Learning rate (synthetic) & $10^{-3}$ \\
Learning rate (real) & $10^{-4}$ \\
Loss function & Huber \\
Gradient clipping & 10.0 (max norm) \\
Shaping weight $\lambda$ & 0.3 (synthetic); 0.0 (real) \\
Synthetic pre-training races & 200 \\
\bottomrule
\end{tabular}
\end{table}

\section{FastF1 Observable Derivation Pipeline}
\label{app:pipeline}

\paragraph{Brake point delta.}
For each sector, we identify the track distance at which the car's speed begins monotonically
decreasing from the sector's speed maximum, using the 10\,Hz speed trace.
A 3-sample Gaussian smoothing is applied before the monotone-descent detection to remove
telemetry noise.

\paragraph{Speed variance detrended.}
The within-sector speed trace is linearly detrended by fitting and subtracting a first-order
polynomial.
The variance of the detrended residual is computed as $\sspeed$.

\paragraph{Active Aero detection.}
In designated activation zones, Active Aero deployment is inferred from the speed residual
after removing the \ERS{}-explained component.
A positive residual above 2.0\,km/h is classified as $\zaero = 1$.

\paragraph{Super-clipping duration ($\dthrottle$).}
For each sector, the fraction of straight-line telemetry samples in which throttle $\geq 98\%$
and speed is simultaneously below the driver's 5-lap rolling speed baseline is computed from
the FastF1 \texttt{car\_data} Throttle channel at 10\,Hz.
If the Throttle channel is absent, $\dthrottle$ defaults to the circuit-appropriate prior
mean for the inferred ERS state.

\section{Synthetic Data Specification}
\label{app:synth}

\subsection{Strategy Profiles}

\begin{table}[h]
\centering
\caption{Synthetic strategy profile definitions.}
\label{tab:profiles}
\begin{tabular}{llp{3cm}p{4cm}}
\toprule
Profile & $p_{\mathrm{burn}}$ schedule & Trap execution & Notes \\
\midrule
Aggressive & 0.75 throughout & Never & Consistently high energy use \\
Conservative & 0.35 throughout & Never & Sustained harvesting \\
Trap-setter & 0.30 baseline; 0.90 post-trap &
  In activation zones when gap $< 1.5$\,s &
  Primary test of Proposition~1; operates from $L_{\mathrm{harvest}}$ \\
Balanced & 0.55 throughout & Never & Average behaviour \\
\bottomrule
\end{tabular}
\end{table}

\subsection{Observation Noise Model}

Synthetic observations are generated from the emission parameters (Section~\ref{sec:emission})
with i.i.d.\ Gaussian noise:
\begin{equation}
\hat{o}_t = o_t + \varepsilon_t, \quad \varepsilon_t \sim \mathcal{N}(0, \sigma^2_{\mathrm{noise}})
\end{equation}
where $\sigma_{\mathrm{noise}} = 0.2$\,km/h for $\vtrap$, $0.05$\,s for $\tsec$, and
1\,m for $\bbrake$.
$\zaero$ is generated directly from Bernoulli emission parameters.
$\sspeed$ is generated from a truncated Gaussian (truncation at~0).
$\dthrottle$ is generated from a truncated Gaussian on $[0,1]$ with the parameters from
Table~\ref{tab:dthrottle}, with a Melbourne circuit offset (+0.10) applied to the
$L_{\mathrm{derate}}$ state to reflect mandatory super-clipping.

Five races are generated per profile (20 total), each of 58 laps $\times$ 3 sectors $\times$
6 rival cars = 1,044 sector observations per race, for 20,880 total observations.

\end{document}